\documentclass[margin=1in, 11pt]{article}
\usepackage{appendix}
\usepackage[figuresright]{rotating}
\usepackage{amsmath,amssymb,amscd,amsthm}
\usepackage{graphicx}
\usepackage{dcolumn}
\usepackage{bm}
\usepackage{ifthen}
\usepackage{rays_defs_18}
\usepackage{commath}

\newcommand{\ignore}[1]{}

\usepackage{hyperref}
\hypersetup{
    bookmarks=true,         
    unicode=false,          
    pdftoolbar=true,        
    pdfmenubar=true,        
    pdffitwindow=false,     
    pdfstartview={FitH},    
    pdfauthor={Reinhard Heckel},     
    pdfsubject={Subject},   
    pdfcreator={Reinhard Heckel},   
    pdfproducer={Producer}, 
    pdfnewwindow=true,      
    colorlinks=true,       
    linkcolor=red,          
    citecolor=green,        
    filecolor=magenta,      
    urlcolor=cyan           
}

\usepackage[
backend=bibtex,
url=false,isbn=false,doi=false,
date=year,
hyperref=auto,
style=alphabetic,
natbib=true,
sorting=nty,
firstinits=true,
maxnames=2, maxbibnames=10,
]{biblatex}

\bibliography{./reference.bib} 

\usepackage{comment}
\usepackage{tikz}
\usepackage{pgfplots}

\usepgfplotslibrary{groupplots} 
\usetikzlibrary{pgfplots.groupplots}

\usepackage{pgfplots,pgfplotstable}
\usepgfplotslibrary{fillbetween}
\pgfplotsset{compat=1.10}
\usepackage{amsmath,amssymb,ifthen,color,framed,bm}

\usetikzlibrary{external}

\usepackage{colortbl}
\definecolor{tblblue}{RGB}{101,124,191}
\definecolor{tblred}{rgb}{1,0.93,0.93}
\definecolor{DarkBlue}{rgb}{0,0,0.7} 
\definecolor{BrickRed}{RGB}{203,65,84}

\usepackage{placeins}

\newtheorem{lemma}{Lemma}

\newtheorem{theorem}{Theorem}

\newtheorem{assumption}{Assumption}

\newcommand\empmean{\mathrm{mean}}
\newcommand\empvar{\mathrm{var}}
\newcommand\empdiag{\mathrm{diag}}
\newcommand\prior{G}
\newcommand\param{\mC}
\newcommand\loss{L}
\newcommand\img{\vx}
\newcommand\real{\text{real}}

\title{Channel Normalization in Convolutional Neural Network avoids Vanishing Gradients}
\author{Zhenwei Dai and Reinhard Heckel}
\date{May 2019}

\begin{document}

\begin{center}

{\bf{\LARGE{
Channel Normalization in Convolutional Neural \\[0.2cm] Networks avoids Vanishing Gradients
}}}

\vspace*{.2in}

{\large{
\begin{tabular}{cccc}
Zhenwei Dai$^{\dagger}$ and Reinhard Heckel$^{\star,\ast}$
\end{tabular}
}}

\vspace*{.05in}

\begin{tabular}{c}
$^\ast$Dept. of Electrical and Computer Engineering 
and Dept. of Statistics$^\dagger$, Rice University \\
$^\star$Dept. of Electrical and Computer Engineering, Technical University of Munich \\
\end{tabular}

\vspace*{.1in}

\today

\vspace*{.1in}

\end{center}

\renewcommand{\Tnorm}{}
\renewcommand\norm[2][\Tnorm]{\ensuremath{{\left\|#2\right\|}_{#1}}}


\begin{abstract}
Normalization layers are widely used in deep neural networks to stabilize training. 
In this paper, we consider the training of convolutional neural networks with gradient descent on a single training example.
This optimization problem arises in recent approaches for solving inverse problems such as the deep image prior or the deep decoder. 
We show that for this setup, channel normalization, which centers and normalizes each channel individually, avoids vanishing gradients, whereas without normalization, gradients vanish which prevents efficient optimization.
This effect prevails in deep single-channel linear convolutional networks, and we show that without channel normalization, gradient descent takes at least exponentially many steps to come close to an optimum. 
Contrary, with channel normalization, the gradients remain bounded, thus avoiding exploding gradients.
\end{abstract}

\section{Introduction}
Deep learning and in particular convolutional neural networks have significantly improved the state-of-the-art in computer vision, image generation, and computational imaging, among many other fields. 
Deep neural networks are typically trained using first order methods such as gradient descent and the stochastic gradient method. However, the corresponding loss function is non-convex and therefore, depending on the initialization, convergence to an optimum is not guaranteed, and first order methods sometimes suffer from unstable training and/or vanishing or exploding gradients.

Normalization layers are widely used to avoid vanishing or exploding gradients, stabilize training, and enable learning with higher rates and faster convergence. The perhaps most popular normalization technique is batch normalization~\cite{ioffe2015batch}; but a number of (often closely related) variations and alternatives have been proposed such as 
layer normalization~\cite{lei2016layer}, 
weight normalization~\cite{salimans2016weight}, and 
instance normalization~\cite{ulyanov2016instance}. 

A variety of recent works have proposed different explanations for the success of normalization layers. 
The original batch normalization paper~\cite{ioffe2015batch} suggested that batch normalization aids optimization by reducing a quantity called internal covariate shift. In contrast, Santurkar et al.~\cite{santurkar2018does} 
reason that batch normalization reparameterizes the underlying optimization problem and thereby make its landscape significantly smoother.
Kohler et al.~\cite{kohler_jonas2018} linked batch normalization to weight normalization~\cite{salimans2016weight}, and pointed out that batch normalization accelerates the training process by splitting the optimization task into optimizing the length and direction of the parameters separately, and Bjorck~\cite{bjorck2018understanding} argues that (batch) normalization enables training with larger training rates. We add that whether normalization layers are useful or not depends strongly on the architecture and initialization. For example carefully initialized deep residual networks can be trained without any normalization layers~\cite{zhang2019fixup}. 

In this paper, we study channel normalization, which is a special case of a number of the above mentioned normalization techniques, in the context of a convolutional generator network. Channel normalization standardizes each channel in a convolutional neural network, individually for each training example, and scales and shifts the resulting vector with a (trainable) scalar. 
Channel normalization is equivalent to instance normalization~\cite{ulyanov2016instance} and to batch normalization for a single training example (then the batch size is one).


We first train a convolutional network with gradient descent on a single training example, a problem that occurs in solving  inverse problems without training data~\cite{heckel2018deep, ulyanov2018deep}, and demonstrate that channel normalization avoids exploding and vanishing gradients and enables reaching a close-to-optimal point.
Contrary, without channel normalization, gradient descent does not converge to an optimum in a reasonable number of iterations.
We then show analytically, for a special case of linear convolutional networks, that without channel normalization, gradient descent requires at least exponentially many steps to converge under mild initialization conditions.



The aforementioned works~\cite{ioffe2015batch,santurkar2018does,kohler_jonas2018} have studied normalization techniques by focusing on shallow networks (i.e., networks with one hidden layer), since analytical gradient expressions for deep networks with non-linearities are almost intractable. 
Here, we sidestep this hurdle by exploring a simpler model, specifically a linear convolutional network with a single channel. 
Studying such a simple model is justified by observing that even for this simple model, normalization is critical for fast convergence.


\section{Channel normalization}
\label{sec:channelnorm}

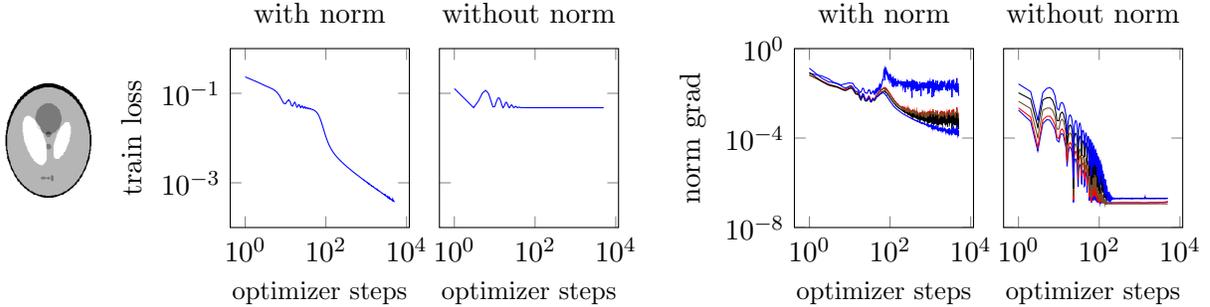
\begin{figure}
\begin{center}
\begin{tikzpicture}

\node at (-2.5,1.1) {\includegraphics[width = 0.12\textwidth]{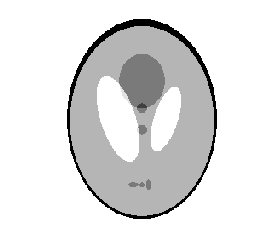} };

\begin{groupplot}[
y tick label style={/pgf/number format/.cd,fixed,precision=3},
scaled y ticks = false,
legend style={at={(1,1)} , font=\tiny ,
/tikz/every even column/.append style={column sep=-0.1cm}
 },
         group
         style={group size= 2 by 1, xlabels at=edge bottom,
         yticklabels at=edge left,
         xticklabels at=edge bottom,
         horizontal sep=0.4cm, vertical sep=0.4cm,
         }, 
           xlabel = {\small optimizer steps},
         width=0.24\textwidth,height=0.24\textwidth,
         ymin = 0.0001,
         ymax = 1,
         ]

\nextgroupplot[title = {with norm},xmode=log,ymode=log, ylabel={train loss}]
	\addplot +[mark=none] table[x index=1,y index=2]{./conv_gen_norm.dat};
\nextgroupplot[title = {without norm},xmode=log,ymode=log]
	\addplot +[mark=none] table[x index=1,y index=2]{./conv_gen_no_norm.dat};
      
\end{groupplot}


\begin{scope}[xshift = 7.5cm]
\begin{groupplot}[
y tick label style={/pgf/number format/.cd,fixed,precision=3},
scaled y ticks = false,
legend style={at={(1,1)} , font=\tiny ,
/tikz/every even column/.append style={column sep=-0.1cm}
 },
         group
         style={group size= 2 by 1, xlabels at=edge bottom,
         yticklabels at=edge left,
         xticklabels at=edge bottom,
         horizontal sep=0.4cm, vertical sep=0.4cm,
         }, 
           xlabel = {\small optimizer steps},
         width=0.24\textwidth,height=0.24\textwidth,
         ymin = 0.00000001,
         ymax = 1,
         ]
\nextgroupplot[title = {with norm},xmode=log,ymode=log, ylabel={norm grad}]
	\addplot +[mark=none] table[x index=1,y index=3]{./conv_gen_norm.dat};
	\addplot +[mark=none] table[x index=1,y index=4]{./conv_gen_norm.dat};
	\addplot +[mark=none] table[x index=1,y index=5]{./conv_gen_norm.dat};
	\addplot +[mark=none] table[x index=1,y index=6]{./conv_gen_norm.dat};
	\addplot +[mark=none] table[x index=1,y index=7]{./conv_gen_norm.dat};
	
\nextgroupplot[title = {without norm},xmode=log,ymode=log]
	\addplot +[mark=none] table[x index=1,y index=3]{./conv_gen_no_norm.dat};
	\addplot +[mark=none] table[x index=1,y index=4]{./conv_gen_no_norm.dat};
	\addplot +[mark=none] table[x index=1,y index=5]{./conv_gen_no_norm.dat};
	\addplot +[mark=none] table[x index=1,y index=6]{./conv_gen_no_norm.dat};
	\addplot +[mark=none] table[x index=1,y index=7]{./conv_gen_no_norm.dat};
      
\end{groupplot}   
\end{scope}

\end{tikzpicture}
\end{center}
\vspace{-0.8cm}
\caption{
\label{fig:phantormMRI}
The train loss and the norm of the gradients in each of the 5 layers of a convolutional neural generator network with and without channel normalization for fitting the phantom MRI image:
Without normalization the gradients vanish before gradient descent reaches a good solution.
}
\end{figure}

We start by introducing channel normalization and then show empirically that it is critical for running gradient descent efficiently on a convolutional generators trained on a single example.

The channel normalization operation normalizes each channel of a convolutional network individually. 
Let $\vz_{ij}$ be the input of the $j$-th channel and the $i$-th layer. Channel normalization performs the transformation
\[
\vz'_{ij} = \frac{ \vz_{ij} - \empmean(\vz_{ij})  }{ \sqrt{ \empvar(\vz_{ij}) +\epsilon}}\gamma_{ij} + \beta_{ij},
\]
where $\empmean$ and $\empvar$ compute the empirical mean and variance, $\gamma_{i,j}$ and $\beta_{ij}$ are parameters learned independently for each channel, and $\epsilon$ is a fixed small constant added for numerical stability.

We consider a variant of the deep decoder introduced in~\cite{heckel2018deep}. 
The network works well for image compression and for regularizing a variety of inverse problems, when trained or fitted to a single image only.
Specifically, we consider an extremely simple convolutional generator consisting of $d=5$ many 3x3 convolutional layers, followed by channel normalization and ReLU activation functions. 
Each layer has $k=32$ channels, and the last layer is a 1x1 convolutional layer mapping the $k$ channels to a single, 256x256 grayscale output image. 
The input to the network is a 32x256x256 volume that is chosen randomly and is fixed (i.e., we do not optimize over the input). 
Given an image $\vx^\ast$ we then fit the parameters of the network (i.e., the weights) by minimizing the loss
$
\loss(\param) = \norm[2]{\prior(\param) - \img^\ast }^2
$
with respect to the network parameters $\param$ using plain gradient descent with fixed stepsize. 

Figure~\ref{fig:phantormMRI} shows the results for the phantom MRI image for a network with and without channel normalization. With channel normalization, the training loss converges rapidly and the gradients do not vanish. Contrary, without normalization, the network does not converge to a small error (even though the network has the capacity to represent the image), and the gradients vanish. 
This effect is not specific to the image (we have reproduced it using 100 randomly chosen images from imagenet), and it is also reproducible for a number of related convolutional generators, for example networks including upsampling operations.


\section{Isolating the effect of channel normalization}
\label{sec:isolating}

We next show that to achieve the stabilizing effect of channel normalization, the trainable coefficients $\gamma_{ij}$ and $\beta_{ij}$ do not need to be learned and can be set to one and zero, respectively. We also demonstrate that even for linear networks, channel normalization is critical to avoid vanishing gradients. This justifies our theoretical study of linear networks in the next section.

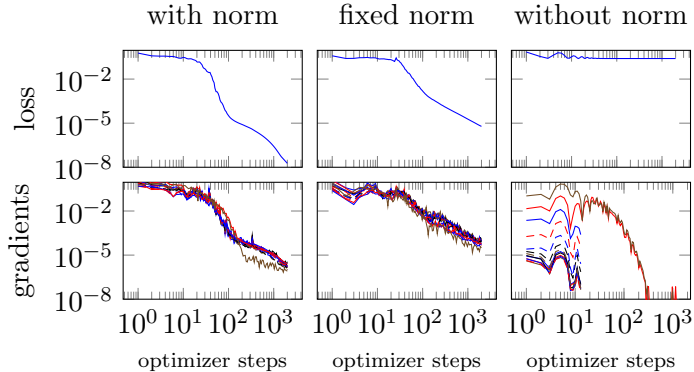
\begin{figure}
\vspace{-0.5cm}
\begin{center}
\begin{tikzpicture}
\begin{groupplot}[
y tick label style={/pgf/number format/.cd,fixed,precision=3},
scaled y ticks = false,
legend style={at={(1,1)} , font=\tiny ,
/tikz/every even column/.append style={column sep=-0.1cm}
 },
         group
         style={group size= 3 by 2, xlabels at=edge bottom,
         yticklabels at=edge left,
         xticklabels at=edge bottom,
         horizontal sep=0.2cm, vertical sep=0.2cm,
         }, 
         width=0.24\textwidth,height=0.19\textwidth,
         ymin = 0.00000001,
         ymax = 0.9,
         ]

\nextgroupplot[title = {with norm},xmode=log,ymode=log, ylabel={loss}]
	\addplot +[mark=none] table[x index=1,y index=2]{./data/BN_grad_mse.dat};
\nextgroupplot[title = {fixed norm}, xmode=log,ymode=log]
	\addplot +[mark=none] table[x index=1,y index=2]{./data/NL_grad_mse.dat};
\nextgroupplot[title = {without norm}, xmode=log,ymode=log]
	\addplot +[mark=none] table[x index=1,y index=2]{./data/LN_grad_mse.dat};

\nextgroupplot[ylabel={gradients}, xmode=log,ymode=log,xlabel={\scriptsize optimizer steps}]
	\addplot +[mark=none] table[x index=1,y index=3]{./data/BN_grad_mse.dat};
	\addplot +[mark=none] table[x index=1,y index=4]{./data/BN_grad_mse.dat};
	\addplot +[mark=none] table[x index=1,y index=5]{./data/BN_grad_mse.dat};
	\addplot +[mark=none] table[x index=1,y index=6]{./data/BN_grad_mse.dat};
	\addplot +[mark=none] table[x index=1,y index=7]{./data/BN_grad_mse.dat};
	\addplot +[mark=none] table[x index=1,y index=8]{./data/BN_grad_mse.dat};
	\addplot +[mark=none] table[x index=1,y index=9]{./data/BN_grad_mse.dat};
	\addplot +[mark=none] table[x index=1,y index=10]{./data/BN_grad_mse.dat};
	\addplot +[mark=none] table[x index=1,y index=11]{./data/BN_grad_mse.dat};
	\addplot +[mark=none] table[x index=1,y index=12]{./data/BN_grad_mse.dat};
	\addplot +[mark=none] table[x index=1,y index=13]{./data/BN_grad_mse.dat};
	\addplot +[mark=none] table[x index=1,y index=14]{./data/BN_grad_mse.dat};
	\addplot +[mark=none] table[x index=1,y index=15]{./data/BN_grad_mse.dat};
\nextgroupplot[xmode=log,ymode=log,xlabel={\scriptsize optimizer steps}]
	\addplot +[mark=none] table[x index=1,y index=3]{./data/NL_grad_mse.dat};
	\addplot +[mark=none] table[x index=1,y index=4]{./data/NL_grad_mse.dat};
	\addplot +[mark=none] table[x index=1,y index=5]{./data/NL_grad_mse.dat};
	\addplot +[mark=none] table[x index=1,y index=6]{./data/NL_grad_mse.dat};
	\addplot +[mark=none] table[x index=1,y index=7]{./data/NL_grad_mse.dat};
	\addplot +[mark=none] table[x index=1,y index=8]{./data/NL_grad_mse.dat};
	\addplot +[mark=none] table[x index=1,y index=9]{./data/NL_grad_mse.dat};
	\addplot +[mark=none] table[x index=1,y index=10]{./data/NL_grad_mse.dat};
	\addplot +[mark=none] table[x index=1,y index=11]{./data/NL_grad_mse.dat};
	\addplot +[mark=none] table[x index=1,y index=12]{./data/NL_grad_mse.dat};
	\addplot +[mark=none] table[x index=1,y index=13]{./data/NL_grad_mse.dat};
	\addplot +[mark=none] table[x index=1,y index=14]{./data/NL_grad_mse.dat};
	\addplot +[mark=none] table[x index=1,y index=15]{./data/NL_grad_mse.dat};
\nextgroupplot[xlabel={\scriptsize optimizer steps},xmode=log,ymode=log]
	\addplot +[mark=none] table[x index=1,y index=3]{./data/LN_grad_mse.dat};
	\addplot +[mark=none] table[x index=1,y index=4]{./data/LN_grad_mse.dat};
	\addplot +[mark=none] table[x index=1,y index=5]{./data/LN_grad_mse.dat};
	\addplot +[mark=none] table[x index=1,y index=6]{./data/LN_grad_mse.dat};
	\addplot +[mark=none] table[x index=1,y index=7]{./data/LN_grad_mse.dat};
	\addplot +[mark=none] table[x index=1,y index=8]{./data/LN_grad_mse.dat};
	\addplot +[mark=none] table[x index=1,y index=9]{./data/LN_grad_mse.dat};
	\addplot +[mark=none] table[x index=1,y index=10]{./data/LN_grad_mse.dat};
	\addplot +[mark=none] table[x index=1,y index=11]{./data/LN_grad_mse.dat};
	\addplot +[mark=none] table[x index=1,y index=12]{./data/LN_grad_mse.dat};
	\addplot +[mark=none] table[x index=1,y index=13]{./data/LN_grad_mse.dat};
	\addplot +[mark=none] table[x index=1,y index=14]{./data/LN_grad_mse.dat};
	\addplot +[mark=none] table[x index=1,y index=15]{./data/LN_grad_mse.dat};
\end{groupplot}

\end{tikzpicture}
\end{center}
\vspace{-0.8cm}
\caption{The training error and the norms of the gradients of each layer verses the number of gradient descent steps: Both channel normalization with fixed scale and bias parameters and learned ones enable efficiently finding an optimum.}
\label{fig2}
\end{figure}

\textbf{Multiple Channels CNN:}
We first consider a one-dimensional convolutional network, again only consisting of convolutional layers followed by channel normalization and ReLU activation functions. 
We set the dimension of input/output vector to $n=256$, the number of channels to $k=4$, number of hidden layers to $d=12$, and convolutional kernel size to $3$. The entries of the input vector $\vx$ are sampled from a standard uniform distribution. 
As before, we minimize the least squares loss with respect to the weight parameters using gradient descent with fixed step size.
We consider three different normalization operations: the original channel normalization, a variant where $\gamma_{ij}=1$ and $\beta_{ij}=0$ (called fixed norm), and no normalization. 
We consider the problem of fitting a simple step function. 
The results reported in Figure~\ref{fig2} show that both normalization versions enable efficient optimization with gradient descent (the training error is near-zero), whereas without channel normalization the training error does not improve after a few iterations and the gradients vanish.

\textbf{Single Channel Linear CNN:} 
Next, we consider an even simpler network with only one channel and without activation functions. Without normalization, the network is linear.  
We set the dimension of input/output vector to $n=64$, number of hidden layers $d=10$, and convolution kernel size $k=9$. 
The results are very similar to the previous experiment in Figure~\ref{fig2} (see Figure~\ref{fig3} in the appendix), and demonstrates the critical role of channel normalization.

We also evaluated the loss function landscape around the point at convergence (Figure~\ref{fig:dist_landscape}). For both multi-channel CNNs and linear CNNs, 
without normalization, the loss function becomes very flat, in comparison to the more steep loss surface pertaining to the case with channel normalization. This indicates that the gradients around the point of convergence are close to $0$, and gradient descent makes little to no progress if the iterates fall into such flat regions.

\begin{figure}[hbt]
\begin{center}
\begin{tikzpicture}[scale=0.50]
\begin{axis}[unbounded coords=jump, view={60}{20},colormap/viridis, title={a) muti-channel, with normalization}]

\addplot3 [surf,shader=faceted] table {./data/Multi_channel_CNN_Norm_Landscape.dat};
\end{axis}

\begin{scope}[xshift = 8cm]
\begin{axis}[unbounded coords=jump, view={60}{20},colormap/viridis,title={b) muti-channel, w/o normalization}]
\addplot3 [surf,shader=faceted] table {./data/Multi_channel_CNN_Linear_Landscape.dat};
\end{axis}
\end{scope}

\begin{scope}[xshift = 16cm]
\begin{axis}[unbounded coords=jump, view={60}{20},colormap/viridis,title={c) single-channel, with normalization}]
\addplot3 [surf,shader=faceted] table {./data/Linear_CNN_Norm_Landscape.dat};
\end{axis}
\end{scope}

\begin{scope}[xshift = 24cm]
\begin{axis}[unbounded coords=jump, view={60}{20},colormap/viridis,title={d) single-channel, w/o normalization}]
\addplot3 [surf,shader=faceted] table {./data/Linear_CNN_Linear_Landscape.dat};
\end{axis}
\end{scope}
\end{tikzpicture}
\end{center}
\vspace{-0.5cm}
\caption{
\label{fig:dist_landscape}
Panels a and b show the landscape around the point of convergence with and without channel normalization for multi-channel CNNs with $d=10$, filter size $9$, and number of channels $4$;
panels c and d contain the same plots but for a single channel linear CNN with $d=10$ and filter size $9$.
}
\end{figure}
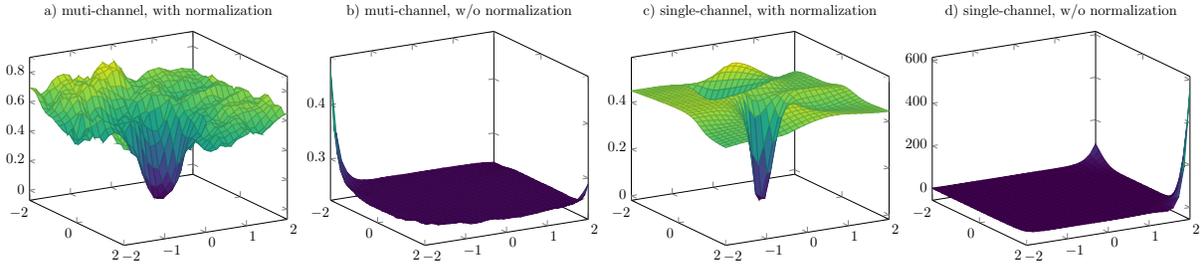

\section{
Theoretical analysis 
}
\label{sec:theory}

From the previous section, we know that channel normalization avoids vanishing gradients even for linear, one-layer convolutional neural networks, and that the scale and shift parameters can be set to one and zero. 
In this section, we provide theoretical justification for the difficulty of optimization in the absence of normalization, and justification for the stabilizing effect of channel normalization.

Throughout this section, we consider a single channel linear convolutional neural network with $d$ layers, with output given as 
$f(\vx,\vw) = \prod_{i=1}^{d} \mW_i \vx$, 
where $\mW_i \in \reals^{n \times n}$ are circulant matrices implementing the convolution operation, and $\vw = (\vw_1,\vw_2,\ldots,\vw_d)$ is the set of weights or convolutional filters, given by the first columns of the respective circulant matrices $\mW_i,\ldots,\mW_{d}$ (which define all other entries of the matrices). 
We study gradient descent applied to the squared loss function $
\loss(\vw,\vx,\vy)
= 
\frac{1}{2}
\norm[2]{\vy - f(\vx,\vw) }^2
$.
We start by showing that without channel normalization gradient descent needs at least exponentially many steps to converge under a standard initialization scheme. 

\begin{theorem}
\label{thm:expmanysteps}
Suppose that the signal $\vy$ doesn't vanish, i.e., $\norm{\vy} / \norm{\vx} \geq d n^{d/2} \tau$, where $\tau$ is a constant. 
Moreover, suppose that $\vy$ is in the range of the generator $f$ and that the initial weights are Gaussian random variables with zero mean and covariance matrix $1/np$ ($p$ is the kernel size of $i$-th convolution layer). Then gradient descent with constant stepsize $\eta \leq \exp(cd)$ runs at least for $\exp(\Omega(d))$ steps until it reaches a point that is $c'$ close to optimal with probability larger than $1 - \exp(-\Omega(d))$. Here, $c$ and $c'$ are constants independent of $d$.
\end{theorem}

The proof, deferred to the appendix, relies on diagonalizing the circulant matrices $\mW_i$ using the Fourier transform. Then, the optimization problem reduces to $n$ one-dimensional problems, and we can build on results by Shamir~\cite{shamir2018exponential} on the hardness of optimizing one-dimensional deep neural networks. We also notice that when the kernel size of convolution layer $p$ is smaller than $n$, gradient descent is only applied to the first $p$ entries of $\vw_i$ and other entries are always fixed to $0$. In this case, the gradient descent on the first $p$ entries of $\vw_i$ is equivalent to the projected gradient descent on all entries of $\vw_i$ while projecting the last $(n-p)$ entries to $0$ during each iteration.

Theorem~\ref{thm:expmanysteps} shows that the number of steps to come close to the optimum is at least exponential in the network depth $d$, even when the stepsize is large (exponential in $d$). This can be interpreted as a case of gradient vanishing. 
Also note that if we initialize the weight away from $0$ with another initialization scheme, e.g., $\vw_k$ is initialized following Gaussian distribution with covariance matrix $\mI$ , then the norms of the gradients increase exponentially fast with the network depth, and the network becomes difficult to optimize due to exploding gradients.

Next, we evaluate the effect of channel normalization on the gradients. Since our experiments have shown that fixing scale and shift parameters or learning them yields comparable performance, we focus on the case where they are fixed to $\gamma_{ij}=1$ and $\beta_{ij}=0$. 
Suppose that the input $\vx$ has zero mean. Then the gradients pertaining to the loss function with normalized loss are (see Appendix~\ref{sec:normlossgrad}):
\begin{align*}
\nabla_{\vw_k} L_N(\mW,\vx,\vy)
=
\transp{\mX}_k
\frac{w_{d+1}}{\norm{ \mX_k \vw_k }} \left(
\mI - \mX_k \vw_k \transp{\vw}_k \transp{\mX}_k  / \norm{\mX_k \vw_k }^2
\right) \vy,
\quad 
\mX_k = \prod_{i\neq k} \mW_i \mX,
\end{align*}
where $\mX$ is the circulant matrix with first column $\vx$, and $w_{d+1}$ is a scale parameter that we optimize over and necessary so that the range of the network can exhaust $\reals^n$. 
By this expression, the gradients are obtained by projecting $\vy$ onto the orthogonal complement of the estimate at the $k$-th iteration, $\mX_k \vw_k$, followed by multiplication with $\transp{\mX}_k/\norm{\mX_k \vw_k}$. In contrast, the norm of the un-normalized gradients is given by $\nabla_{\vw_k} L(\mW,\vx,\vy) = \transp{\mX}_k (\vy - \mX_k \vw_k)$. In Figure~\ref{fig:dist_gradnorms} Panel (a) and (b) we plot the distribution with and without normalization at initialization for a network with $n=100$ and $d=6$ layers. 
Note that the loss typically diverges at the first few iterations, which justifies considering the gradients at initialization. The results show that the normalization leads to the gradients to be significantly better behaved, i.e., the distribution does have a significantly smaller tail.

In Figure~\ref{fig:dist_gradnorms} Panel (c) and (d) we plot the distribution for a multi-channel CNN with ReLU activation functions, and likewise, the results shows that without channel normalization, the tail is significantly larger. 
Thus, without normalization for a given stepsize the network is much more susceptible to vanishing or exploding gradients.

\definecolor{DarkBlue}{rgb}{0,0,0.7} 
\definecolor{BrickRed}{RGB}{203,65,84}

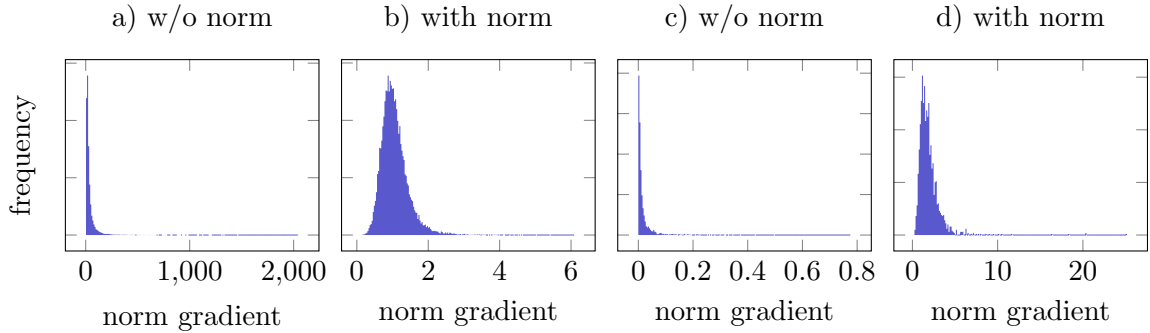
\begin{figure}[h!]
\begin{center}
\begin{tikzpicture}

\begin{groupplot}[
y tick label style={color=white, draw = none},
scaled y ticks = false,
         group
         style={group size=4 by 3, 
         ylabels at={edge left}, 
         yticklabels at=edge left,
         horizontal sep=0.3cm,vertical sep=2.3cm}, xlabel={norm gradient},
         width=0.3\textwidth,
         height=0.25\textwidth
]
	\nextgroupplot[title = {a) w/o norm}, ylabel={frequency}, yticklabels={,,}]
	\addplot +[ycomb,DarkBlue!65,mark=none] table[x index=0,y index=1]{./data/hist_nonorm.data};

	\nextgroupplot[title = {b) with norm}]
	\addplot +[ycomb,DarkBlue!65,mark=none] table[x index=0,y index=1]{./data/hist_norm.data};
	
	\nextgroupplot[title = {c) w/o norm}]
	\addplot +[ycomb,DarkBlue!65,mark=none] table[x index=0,y index=1]{./data/hist_CNN_linear_b200.data};
	
	\nextgroupplot[title = {d) with norm}]
	\addplot +[ycomb,DarkBlue!65,mark=none] table[x index=0,y index=1]{./data/hist_CNN_norm_b200.data};

\end{groupplot}          
\end{tikzpicture}
\end{center}

\vspace{-0.5cm}
\caption{
\label{fig:dist_gradnorms}
Panel a) and b)  show the distribution of the gradients at initialization with normalization and without normalization for a single channel linear CNN with $n=100$ and $d=6$; c) and d) show that for a multiple channel CNN with $n=64$, $d=6$, filter size $9$ and number of channels $4$.
}
\end{figure}

\newpage

\section*{Code}
Code to reproduce the experiments is available at \href{https://github.com/reinhardh/normalization_dnns}{github.com/reinhardh/normalization\_dnns}.

\section*{Acknowledgements}
RH and ZD are partially supported by NSF award IIS-1816986.

\printbibliography


\newpage

\appendix

\section*{Appendix}

\section{Convergence for linear single layer convolutional networks}

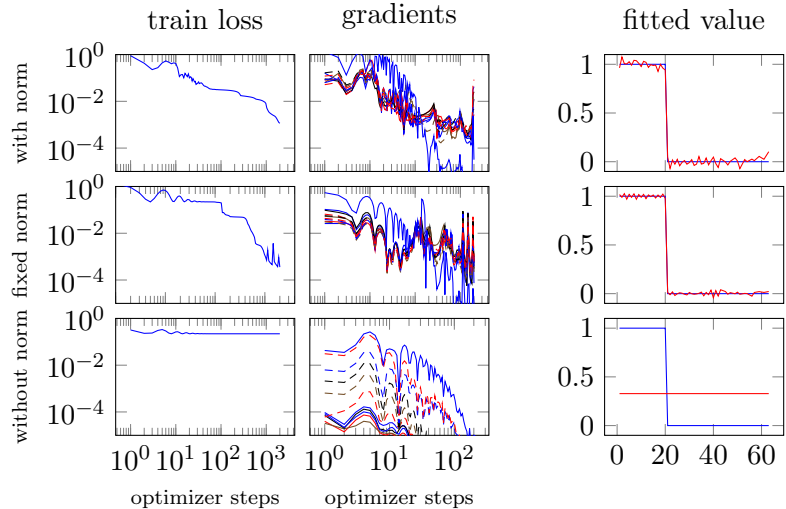
\begin{figure}[hbt]
\begin{center}
\begin{tikzpicture}
\begin{groupplot}[
y tick label style={/pgf/number format/.cd,fixed,precision=3},
scaled y ticks = false,
legend style={at={(1,1)} , font=\tiny ,
/tikz/every even column/.append style={column sep=-0.1cm}
 },
         group
         style={group size= 2 by 3, xlabels at=edge bottom,
         yticklabels at=edge left,
         xticklabels at=edge bottom,
         horizontal sep=0.2cm, vertical sep=0.2cm,
         }, 
         width=0.24\textwidth,height=0.19\textwidth,
         ymin = 0.00001,
         ymax = 1,
         ]

\nextgroupplot[title = {train loss},xmode=log,ymode=log, ylabel={\scriptsize with norm}]
	\addplot +[mark=none] table[x index=1,y index=2]{./data/BN_grad_mse2.dat};
\nextgroupplot[title = {gradients}, xmode=log,ymode=log]
	\addplot +[mark=none] table[x index=1,y index=3]{./data/BN_grad_mse2.dat};
	\addplot +[mark=none] table[x index=1,y index=4]{./data/BN_grad_mse2.dat};
	\addplot +[mark=none] table[x index=1,y index=5]{./data/BN_grad_mse2.dat};
	\addplot +[mark=none] table[x index=1,y index=6]{./data/BN_grad_mse2.dat};
	\addplot +[mark=none] table[x index=1,y index=7]{./data/BN_grad_mse2.dat};
	\addplot +[mark=none] table[x index=1,y index=8]{./data/BN_grad_mse2.dat};
	\addplot +[mark=none] table[x index=1,y index=9]{./data/BN_grad_mse2.dat};
	\addplot +[mark=none] table[x index=1,y index=10]{./data/BN_grad_mse2.dat};
	\addplot +[mark=none] table[x index=1,y index=11]{./data/BN_grad_mse2.dat};
	\addplot +[mark=none] table[x index=1,y index=12]{./data/BN_grad_mse2.dat};
	\addplot +[mark=none] table[x index=1,y index=13]{./data/BN_grad_mse2.dat};
\nextgroupplot[ylabel = {\scriptsize fixed norm}, xmode=log,ymode=log]
	\addplot +[mark=none] table[x index=1,y index=2]{./data/NL_grad_mse2.dat};
\nextgroupplot[xmode=log,ymode=log]
	\addplot +[mark=none] table[x index=1,y index=3]{./data/NL_grad_mse2.dat};
	\addplot +[mark=none] table[x index=1,y index=4]{./data/NL_grad_mse2.dat};
	\addplot +[mark=none] table[x index=1,y index=5]{./data/NL_grad_mse2.dat};
	\addplot +[mark=none] table[x index=1,y index=6]{./data/NL_grad_mse2.dat};
	\addplot +[mark=none] table[x index=1,y index=7]{./data/NL_grad_mse2.dat};
	\addplot +[mark=none] table[x index=1,y index=8]{./data/NL_grad_mse2.dat};
	\addplot +[mark=none] table[x index=1,y index=9]{./data/NL_grad_mse2.dat};
	\addplot +[mark=none] table[x index=1,y index=10]{./data/NL_grad_mse2.dat};
	\addplot +[mark=none] table[x index=1,y index=11]{./data/NL_grad_mse2.dat};
	\addplot +[mark=none] table[x index=1,y index=12]{./data/NL_grad_mse2.dat};
	\addplot +[mark=none] table[x index=1,y index=13]{./data/NL_grad_mse2.dat};
\nextgroupplot[ylabel = {\scriptsize without norm}, xlabel={\scriptsize optimizer steps}, xmode=log,ymode=log]
	\addplot +[mark=none] table[x index=1,y index=2]{./data/LN_grad_mse2.dat};
\nextgroupplot[xlabel={\scriptsize optimizer steps},xmode=log,ymode=log]
	\addplot +[mark=none] table[x index=1,y index=3]{./data/LN_grad_mse2.dat};
	\addplot +[mark=none] table[x index=1,y index=4]{./data/LN_grad_mse2.dat};
	\addplot +[mark=none] table[x index=1,y index=5]{./data/LN_grad_mse2.dat};
	\addplot +[mark=none] table[x index=1,y index=6]{./data/LN_grad_mse2.dat};
	\addplot +[mark=none] table[x index=1,y index=7]{./data/LN_grad_mse2.dat};
	\addplot +[mark=none] table[x index=1,y index=8]{./data/LN_grad_mse2.dat};
	\addplot +[mark=none] table[x index=1,y index=9]{./data/LN_grad_mse2.dat};
	\addplot +[mark=none] table[x index=1,y index=10]{./data/LN_grad_mse2.dat};
	\addplot +[mark=none] table[x index=1,y index=11]{./data/LN_grad_mse2.dat};
	\addplot +[mark=none] table[x index=1,y index=12]{./data/LN_grad_mse2.dat};
	\addplot +[mark=none] table[x index=1,y index=13]{./data/LN_grad_mse2.dat};
\end{groupplot}   

\begin{scope}[xshift = 6.5cm]
\begin{groupplot}[
y tick label style={/pgf/number format/.cd,fixed,precision=3},
scaled y ticks = false,
legend style={at={(1,1)} , font=\tiny ,
/tikz/every even column/.append style={column sep=-0.1cm}
 },
         group
         style={group size= 1 by 3, xlabels at=edge bottom,
         yticklabels at=edge left,
         xticklabels at=edge bottom,
         horizontal sep=0.6cm, vertical sep=0.2cm,
         }, 
        width=0.24\textwidth,height=0.19\textwidth,
         ymin = -0.1,
         ymax = 1.1,
         ]
\nextgroupplot[title = {fitted value}]
	\addplot +[mark=none] table[x index=1,y index=2]{./data/BN_fitted2.dat};
	\addplot +[mark=none] table[x index=1,y index=3]{./data/BN_fitted2.dat};
\nextgroupplot[]
	\addplot +[mark=none] table[x index=1,y index=2]{./data/NL_fitted2.dat};
	\addplot +[mark=none] table[x index=1,y index=3]{./data/NL_fitted2.dat};
\nextgroupplot[]
	\addplot +[mark=none] table[x index=1,y index=2]{./data/LN_fitted2.dat};
	\addplot +[mark=none] table[x index=1,y index=3]{./data/LN_fitted2.dat};
      
\end{groupplot}   
\end{scope}
\end{tikzpicture}
\end{center}
\caption{The left panel shows the training error and gradients of different layers over the number of gradient descent steps for optimizing a linear network (no activation functions) with a single channel in each layer only.
The right panel shows the fitted response (red) and true response (blue). The results show that normalization is critical for reaching a good minima, and that both normalization with trained and fixed scale and shift parameters works similarly well.
}
\label{fig3}
\end{figure}


\section{Proof of Theorem~\ref{thm:expmanysteps}}

Our proof relies on diagonalizing the circulant matrices implementing the convolutions with the Fourier transform. 

\subsection{Linear single channel CNNs in the Fourier domain}

Since the matrices $\mW_i$ are circulant, they can be diagonalized with the Fourier transformation. 
As a consequence, the loss function of a linear CNN becomes a sum of loss functions of one dimensional single channel deep linear neural networks. 

With $\mW_i = \mF\empdiag(\sqrt{n}\mF^{H}\vw_i) \mF^{H}$, where $\mF$ is the unitary $n \times n$ discrete Fourier transform matrix and $\vw_i$ is the first column of $\mW_i$, the network's output without normalization can be expressed as
\begin{align*}
f(\vx,\mW) 
= \prod_{i=1}^{d}\mW_i \vx 
= \prod_{i=1}^{d}\mF\empdiag(\sqrt{n}\mF^{H}\vw_i)\mF^H \vx 
= n^{d/2} \mF\left(\prod_{i=1}^{d} \empdiag(\mF^{H}\vw_i)\right)\mF^H \vx.
\end{align*}
With this expression, the loss function becomes
\begin{align*}
L(\mW,\vx,\vy) 
= 
\frac{1}{2}\norm{\vy-f(\vx,\mW)}^2 
&= \frac{1}{2} \norm{\vy - n^{d/2} \mF \left(\prod_{i=1}^{d} \empdiag(\mF^{H}\vw_i)\right)\mF^H \vx}^2  \\
&= 
\frac{1}{2} \norm{\mF^H \vy - n^{d/2} \left(\prod_{i=1}^{d} \empdiag(\mF^{H}\vw_i)\right)\mF^H \vx}^2  \nonumber \\
&=
\frac{1}{2} \sum_{j=1}^{n} \abs{\vf_j^H \vy - n^{d/2} \left(\prod_{i=1}^{d} \vf_j^{H}\vw_i \right)\vf_j^H \vx}^2,\nonumber 
\end{align*}
where $\vf_j$ is the $j$-th column of the Fourier matrix $\mF$.


\subsection{Proof of Theorem~\ref{thm:expmanysteps}}

In this section, we show that gradient descent takes exponentially many steps to converge under Xavier initialization, a standard initialization scheme.
The proof follows a similar line of arguments as a very related result by Shamir~\cite{shamir2018exponential} on the hardness of optimizing one-dimensional deep neural networks.

Assume the kernel size of the $i$-th convolution layer is $p \geq 1$. Thus, $\vw_i \in \mathbb{S}_p$, where $\mathbb{S}_p = \{\vz=(z_1, z_2,\ldots,z_n)| z_1,\ldots,z_{p} \in \mathbb{R}, z_{p+1},\ldots,z_n =0 \}$. We further denote $\vw \in \mathbb{S}_{p \times d}$, where $\mathbb{S}_{p \times d} = \{\vu = (\vu_1, \vu_2, \ldots, \vu_d)| \vu_1, \vu_2, \ldots, \vu_d \in \mathbb{S}_p \}$.

\begin{assumption}
\label{ass:xavier}
Assume $\vw_1, \vw_2, \ldots, \vw_d$ are drawn independently and the first $p$ entries of $\vw_i$ are drawn i.i.d from a distribution that satisfies 
\begin{align*}
&\PR{ \norm{\vw_i} \leq t } \leq c_1 t \quad \text{and} \\
&\EX{ \norm{\vw_i} }\leq \frac{1}{\sqrt{n}} (1 - c_2),
\end{align*}
where the constants $c_1, c_2 >0$ are independent of $d$.
\end{assumption}

The assumption holds for some widely used initialization distributions, like the distribution $\mc N(0, \frac{1}{np})$ in the statement of the theorem, which follows from $\EX{\norm{\vw_i}} \leq (1-c_2)/\sqrt{n}$ with $c_2 = 1 - \sqrt{\frac{2}{p}}\frac{\Gamma((p+1)/2)}{\Gamma(p/2)}>0$ (since $\EX{\norm{\vw_i}} = \sqrt{\frac{2}{np}}\frac{\Gamma((p+1)/2)}{\Gamma(p/2)}$).

Next, we show that with the initialization satisfying Assumption~\ref{ass:xavier}, gradient descent takes at least exponentially many gradient descent iterations to reach a close-to-optimal point with high probability. 

The key idea of the proof of Theorem~\ref{thm:expmanysteps} is to show that if we start from such a random initialization 
$\vw= (\vw_1,\vw_2,\ldots, \vw_d)$, then gradient descent must take exponentially many steps to escape a ball of radius $r$ around $\vw$, defined as
\[
\mc B(\vw,r) 
= 
\left\{\vv = (\vv_1,\vv_2,\ldots, \vv_d)
\Big|
\sum_{i=1}^{d}\norm{\vv_i - \vw_i}^2 \leq r^2 \right\}.
\]
We then show that the loss function value evaluated at any point inside the ball is sub-optimal, i.e., for $\vv \in \mc B(\vw,r)$, $L(\vv) \geq c'$, where $c'$ is a constant. This will establish the proof.

To show that there exist a radius $r>0$ such that gradient descent takes at least exponentially many steps to escape the ball $\mc B(\vw,r)$. We notice that for any $\vv \in \mathbb{S}_{p \times d}$, running gradient descent on the non-zero entries of $\vv_k$ is equivalent to performing projected gradient descent on $\vv_k$ while constraining $\vv_k \in \mathbb{S}_p$. Hence, after each iteration, $\vv_k$ is updated by $\Pi_{\mathbb{S}_p}(\vv_k - \eta \nabla_{\vv_k}L(\vv))$, where $\eta$ is the stepsize and $\Pi$ is the orthogonal projection operator. Moreover, since $\mathbb{S}_{p}$ is a closed and convex set, the distance between $\vv^{(t)}_k$ and $\vv^{(t+1)}_k$ after $t$-th iteration can be upper bounded by 
\begin{eqnarray}
\norm{\vv^{(t+1)}_k - \vv^{(t)}_k}
=
\norm{\Pi_{\mathbb{S}_p}(\vv^{(t)}_k - \eta \nabla_{\vv_k} L(\vv^{(t)})) - \vv^{(t)}_k} 
&=& 
\norm{\Pi_{\mathbb{S}_p}(\vv^{(t)}_k - \eta \nabla_{\vv_k} L(\vv^{(t)})) - \Pi_{\mathbb{S}_p}(\vv^{(t)}_k)}   \nonumber \\
&\leq&
\eta \norm{\nabla_{\vv_k}L(\vv^{(t)})}.  \nonumber
\end{eqnarray}
Accordingly, $\vv^{(t)}$ is updated at most $\eta \norm{\nabla L(\vv^{(t)})}$ after $t$-th iteration. Therefore, the number of iterations required to escape a ball of radius $r$ is at least 
$r /(\eta \sup_{ \vv \in \mc B(\vw,r) } \norm{\nabla L(\vv)})$, since at each step, gradient descent can at most move by $\eta \sup_{ \vv \in \mc B(\vw,r)} \norm{\nabla L(\vv)}$.

The following lemma provides an upper bound on $\norm{\nabla L(\vv)}$, which enables us to show that the number of iterations must be large.

\begin{lemma}
\label{lem:grad}
Suppose that the initial point $\vw$ satisfies, for some 
$\alpha,\delta >0$, 
i) $\max_k \left(\prod_{i \neq k} \norm{\vw_i}\right) \leq \alpha$ 
and 
ii) $\min_i \norm{\vw_i} \geq \delta$.
Then, there exists a radius $r$ such that for all $\vv \in \mc B(\vw,r)$, it holds that 
$\prod_{i=1}^{d} \norm{\vv_i} \leq \alpha \exp\left(\frac{\sqrt{d}r}{\delta}\right) \max_k \norm{\vv_k}$ 
and
$\displaystyle \norm{\nabla L(\vv)} \leq
\norm{D(\vv)} \norm{\vx}\alpha \sqrt{d} \exp\left(\frac{\sqrt{d}}{\delta}r\right)$, where 
$L(\vv) = \frac{1}{2}\norm{D(\vv)}^2$ and $D(\vv) = \mF^H\vy-n^{d/2}\prod_{i=1}^{d}\empdiag(\mF^H\vv_i)\mF^H\vx$.
\end{lemma}

Lemma~\ref{lem:grad} guarantees that in a ball $\mc B(\vw,r)$ around the initialization, the gradients of the loss function are strictly upper bounded. So, provided the stepsize is not too large, the progress made in each step of gradient descent is also upper bounded. Then, we show with high probability, that there is a radius $r$ that is much larger than the updates in each step. 

Evoking Lemma~\ref{lem:grad}, given a constant stepsize $\eta$, the number of steps required to escape the ball $\mc B(\vw,r)$ is at least $\frac{r}{\Omega(\eta n^{d/2} \alpha \sqrt{d} \exp(\frac{\sqrt{d}}{\delta}r))}$ 
if $\norm{D(\vv)} \norm{\vx}$ is upper bounded by a numerical constant. 
Next, we show that conditions i and ii from Lemma~\ref{lem:grad} hold with high probability for $\alpha = \frac{\exp(-2cd)}{n^{d/2}}$ and $\delta = \exp(-cd)$ given Assumption~\ref{ass:xavier}. 
We then show that there is a radius $r$ in which $\norm{D(\vv)}$ is upper bounded and gradient descent takes at least exponentially many steps to escape the ball $\mc B(\vw,r)$. 

We start by showing that conditions i and ii from Lemma~\ref{lem:grad} hold with high probability for $\alpha = \frac{\exp(-2cd)}{n^{d/2}}$ and $\delta = \exp(-cd)$.

\begin{lemma}
\label{lem:condi-iii}
Suppose $\vw$ is initialized satisfying Assumption~\ref{ass:xavier}. With probability at least $1 - \Omega(d e^{-cd})$, the conditions
\[
\text{i) 
$\max_k \prod_{i \neq k}\norm{\vw_i} \leq \frac{\exp(-2cd)}{n^{d/2}}$, 
ii) $\min_i \norm{\vw_i} \geq \exp(-cd)$, and
iii) $\max_i \norm{\vw_i} \leq \exp(cd)$
}
\]
hold simultaneously, where $c$ is a constant independent of $d$.
\end{lemma}

\begin{proof}
With Markov's inequality and Assumption~\ref{ass:xavier}, we have, for $t>0$, that
\begin{align*}
\PR{
\prod_{i \neq k}\norm{\vw_i} \geq t
}
&\leq
\frac{\prod_{i \neq k} \EX{ \norm{\vw_i}}}{t}
\leq 
\frac{((1-c_2)/\sqrt{n})^{d-1}}{t}.
\end{align*}
Let $c$ be a constant so that $\exp(-4c) = 1-c_2$ and set $t = \frac{\exp(-2cd)}{n^{d/2}}$. 
Then we obtain
\begin{align}
\PR{
\prod_{i \neq k}\norm{\vw_i} \geq \frac{\exp(-2cd)}{n^{d/2}}
}
\leq
\frac{\sqrt{n}\exp(-4c(d-1))}{\exp(-2cd)} = \Omega(\exp(-2cd)).  
\end{align}
Thus, by the union bound, 
$\max_k \prod_{i\neq k} \norm{\vw_i} \leq \frac{\exp(-2cd)}{n^{d/2}}$ with probability at least $1- \Omega(d\exp(-2cd))$.

We next consider $\min_i \norm{\vw_i}$. 
Again by Assumption~\ref{ass:xavier}, it holds for all $i$ that $\PR{\norm{\vw_i}\leq \exp(-cd)} \leq \Omega(\exp(-cd))$. Then, by the union bound, 
\begin{align}
\PR{\min_i \norm{\vw_i} \leq \exp(-cd)}
\leq 
\sum_{i=1}^{d} 
\PR{
\norm{\vw_i} \leq \exp(-cd)
}
\leq \Omega(d\exp(-cd)).
\end{align}
Finally, by the union bound, Markov's inequality, and Assumption~\ref{ass:xavier}, we have that
\begin{align}
\PR{
\max_i \norm{\vw_i} \geq \exp(cd)
}
\leq \sum_{i=1}^{d} \PR{\norm{\vw_i} \geq \exp(cd)} 
\leq
\Omega(d\exp(-cd)).
\end{align}
Thus, by a union bound, with probability at least 
\[
1-\Omega(\exp(-2cd)) - \Omega(d\exp(-cd)) - \Omega(d\exp(-cd)) = 1-\Omega(d\exp(-cd)),
\]
the conditions i-iii are satisfied simultaneously, which concludes the proof.
\end{proof}

Next, we show that under conditions i-iii in Lemma~\ref{lem:condi-iii}, provided that $d$ is large enough, $\norm{D(\vv)}$ is upper bounded by a constant independent of $d$ for  radius $r = \Omega(\frac{\delta}{\sqrt{d}}) = \Omega(\frac{\exp(-cd)}{\sqrt{d}})$.

Evoking Lemma~\ref{lem:grad}, we have
\begin{align*}
\prod_{i=1}^{d} \norm{\vv_i} 
&\leq 
\alpha\exp\left(\frac{\sqrt{d}r}{\delta}\right)  \max_k \norm{\vv_k} \\
&\leq
\frac{\exp(-2cd)}{n^{d/2}} \Omega(\exp(cd)) 
\max_k \norm{\vv_k} \\
&\leq \frac{\Omega(1)}{n^{d/2}},
\end{align*}
where the second inequality follows from our choice of $\alpha$ and 
$\exp(\frac{\sqrt{d}r}{\delta}) = \Omega(1)$ by our choice of $\delta$, 
and the third inequality follows from 
$\max_k \norm{\vv_k} \leq \max_k \norm{\vw_k} + r = \Omega(\exp(cd))$.

Then, we can upper bound $\norm{D(\vv)}$ when $d$ is large,

\begin{eqnarray}
\norm{D(\vv)} 
&=& \norm{\mF^H\vy-n^{d/2}\prod_{i=1}^{d}\empdiag(\mF^H\vv_i)\mF^H\vx}
\leq \sum_{j=1}^{n} \left(\abs{\vf^H_j \vy} + n^{d/2} \prod_{i=1}^{d} \abs{\vf^H_j\vv_i} \abs{\vf^H_j\vx} \right)   \nonumber     \\
&\leq& n \left(\norm{\vy} + n^{d/2} \prod_{i=1}^{d} \norm{\vv_i} \norm{\vx} \right)   
\leq n \left(\norm{\vy} + n^{d/2} \norm{\vx} \right).   \nonumber 
\end{eqnarray}
Therefore, $\norm{D(\vv)}$ is upper bounded by a constant (given $\vy$ is constant, and $\vx$ satisfies the assumption in theorem~\ref{thm:expmanysteps} independent of $d$.

Finally, we can prove that gradient descent takes at least exponentially many steps to escape $\mc B(\vw,r)$ with $r = \Omega(\frac{\exp(-cd)}{\sqrt{d}})$. By lemma~\ref{lem:condi-iii}, with conditions i-iii satisfied, the number of steps in $\mc B(\vw,r)$ is at least $\frac{r}{\norm{D(\vv)} \norm{\vx} \Omega(\eta n^{d/2} \alpha \sqrt{d} \exp(\frac{\sqrt{d}}{\delta}r))}$. 
Since $\norm{D(\vv)}$ can be upper bounded by a constant, and using that $\norm{\vx}$ is upper bounded by a constant (from $\norm{\vx} < \norm{\vy}$, with $\norm{\vy}$ upper bounded by a constant) and $\eta \leq \exp(\frac{d}{2})$ (by the assumptions in theorem~\ref{thm:expmanysteps}), $\frac{r}{\norm{D(\vv)} \norm{\vx} \Omega(\eta n^{d/2} \alpha \sqrt{d} \exp(\frac{\sqrt{d}}{\delta}r))} \geq \frac{r}{\Omega(\exp(-\frac{3}{2}cd + \frac{\sqrt{d}}{\delta}r))} = \Omega(\exp(\frac{1}{2}cd))$, which increases exponentially in $d$.

It remains to prove that for any $\vv \in \mc B(\vw,r)$ (with $r$ as chosen above), the loss function $L(\vv)$ is lower bounded away from the global minimum. We have proved that when $d$ is large, $\prod_{i=1}^{d} \abs{\vf^H_j \vv_i} \leq \prod_{i=1}^{d} \norm{\vv_i} \leq \Omega(1)$. From the assumption
$\norm{\vy} - \norm{\vx} dn^{d/2} \geq \tau$, there exists $t \in [n]$ such that 
$\abs{\vf^H_t \vy} - n^{d/2} \abs{\vf^H_t \vx} \geq \tau$. So, we have 
\begin{align}
\abs{\vf^H_t \vy - n^{d/2} \prod_{i=1}^{d} \left(\vf^H_t \vv_i\right) \vf^H_t \vx} &\geq \abs{\vf^H_t \vy} - \abs{n^{d/2} \prod_{i=1}^{d} \left(\vf^H_t \vv_i\right) \vf^H_t \vx} \nonumber \\
&\geq \abs{\vf^H_t \vy} - n^{d/2}  \prod_{i=1}^{d} \abs{\vf^H_t \vv_i} \abs{\vf^H_t \vx} \nonumber  \\
&\geq \abs{\vf^H_t \vy} - n^{d/2} \abs{\vf^H_t \vx} > \tau. \nonumber 
\end{align}
Thus, for $\vv \in \mc B(\vw,r)$, the loss function obeys 
$L(\vv) = \sum_{j=1}^{n} \abs{\vf^H_j \vy - n^{d/2} \prod_{i=1}^{d} \left(\vf^H_j \vv_i\right) \vf^H_j \vx} > \tau$ and is thus lower bounded away from zero-training error. 


\subsection{Proof of Lemma~\ref{lem:grad}}

Let $\vv \in \mc B(\vw,r)$ as defined previously. We have 
\begin{align}
\norm{\nabla L(\vv)}^2 
&= 
\sum_{k=1}^{d} 
\norm{\nabla_{\vv_k} L(\vv)}^2
=
\sum_{k=1}^{d}
n^d \norm{D(\vv)}^2
\left(\prod_{i \neq k}\norm{\vv_i}\right)^2 \norm{\vx}^2 \nonumber \\
&\leq
\sum_{k=1}^{d}
n^d \norm{D(\vv)}^2
\left( \max_k \prod_{i \neq k}\norm{\vv_i}\right)^2 \norm{\vx}^2.
\label{eq:adf}
\end{align}
Here, we used that
\[
\norm{
\nabla_{\vv_k} L(\vv)
}
=
\norm{
n^{d/2}
\norm{D(\vv)}
\diag(\herm{\mF} \vx_i )
\prod_{i\neq k}
\diag(\herm{\mF} \vv_i )
}
\leq
n^{d/2} 
\norm{D(\vv)}
\norm{\vx}
\prod_{i\neq k}
\norm{\vv_i}.
\]
The lemma now follows from $\max_k \left(\prod_{i \neq k}\norm{\vv_i}\right) \leq \left(\alpha \sqrt{d} \right) \exp\left(\frac{\sqrt{d}}{\delta}r\right)$. 

Define $\vr_i = \vv_i- \vw_i$, for notational convenience, and note that 
$\sum_{i=1}^{d} \norm{\vr_i}^2 \leq r^2$
since $\vv \in \mc B(\vw,r)$. Then, we have
\begin{align}
\prod_{i \neq k} \norm{\vv_i} 
&= \prod_{i \neq k} \norm{\vw_i + \vr_i}  \nonumber \\
&\leq \left(\prod_{i \neq k} \norm{\vw_i}\right)\left[\prod_{i \neq k}\left(1 + \frac{\norm{\vr_i}}{\norm{\vw_i}} \right)\right] \nonumber \\
&\leq \left(\prod_{i \neq k} \norm{\vw_i}\right)\left[\prod_{i \neq k}\left(1 + \frac{\norm{\vr_i}}{\delta} \right)\right]  \nonumber \\
&= 
\left(\prod_{i \neq k} \norm{\vw_i}\right)\exp\left[\sum_{i \neq k} \log\left(1 + \frac{\norm{\vr_i}}{\delta} \right)\right] \leq \left(\prod_{i \neq k} \norm{\vw_i}\right)\exp\left[\sum_{i \neq k}  \frac{\norm{\vr_i}}{\delta} \right],
\label{eq4}
\end{align}
where the second inequality follows from the assumption $\min_i \norm{\vw_i} \geq \delta$. 
We also have 
\begin{align}
r^2 \geq \sum_{i=1}^{d} \norm{\vr_i}^2 \geq \frac{\left(\sum_{i=1}^{d}\norm[1]{\vr_i}\right)^2}{d}.
\label{eq5}
\end{align}
Inserting equation~\eqref{eq5} into equation~\eqref{eq4}, we get
\begin{align}
\prod_{i \neq k}\norm{\vv_i} 
\leq 
\left(\prod_{i \neq k} \norm{\vw_i}\right)\exp\left(\frac{\sqrt{d}r}{\delta} \right) 
\leq
\alpha \exp\left(\frac{\sqrt{d}r}{\delta} \right),  \nonumber
\end{align}
where the last inequality follows from the assumption 
$\max_k \left(\prod_{i \neq k} \norm{\vw_i}\right) \leq \alpha$. Therefore, $\prod_{i=1}^{d} \norm{\vv_i} \leq \alpha\exp\left(\frac{\sqrt{d}r}{\delta}\right) \cdot \max_k\norm{\vv_k}$. Application of this inequality in~\eqref{eq:adf} yields $\norm{\nabla L(\vv)} \leq \sqrt{d} \cdot \max_k \norm{\nabla_k L(\vv)} \leq n^{d/2} \norm{D(\vv)}\norm{\vx}\left(\alpha \sqrt{d} \right) \exp \left(\frac{\sqrt{d}}{\delta}r\right)$, which concludes the proof.



\section{Channel normalization on a single channel linear CNN}
\label{sec:normlossgrad}

We consider channel normalization with scale parameter fixed to $\gamma_{ij}=1$ and the shift parameter fixed to $\beta_{ij}=0$. Hence, merely varying the scale of the parameters does not change the output of the network with channel normalization $f_N(\mW,\vx)$. 

To simplify the derivation we assume the input vector $\vx$ is centered, i.e., its entries sum to zero. Let $\vz_i$ be the input of the $i$-th layer. Then, the output of the convolutional layer can be written as $\mW_i \vz_i$, where $\mW_i$ is the circulant matrix implementing the convolution operation. Channel normalization centers the mean and adjusts the empirical variance to one. Since $\mW_i$ is a convolution operation, given $\vz_i$ is centered, $\mW_i \vz_i$ is centered as well. 
Thus, the effect of channel normalization in this setup is to normalize the scale of the vector.

It follows that the output of the network with channel normalization can be written as 
\begin{align}
\label{eq:expfn}
f_N(\mW,\vx)    
= 
w_{d+1} \frac{\prod_{i=1}^{d}\mW_i\vx}{\norm{\prod_{i=1}^{d} \mW_i \vx}},
\end{align}
where $w_{d+1}$ is a scale parameter that we introduced so that $f_N(\mW,\vx)$ can exhaust $\reals^n$.

To see that the output of the network can be written as in equation~\eqref{eq:expfn}, note that the input of the first layer is $\vz_1 = \vx$, with $\vx$ centered. The normalization of the first layer yields the input of the second layer by division by $\norm{\mW_1 \vz_1}/\sqrt{n} = \norm{\mW_1 \vx}/\sqrt{n}$, which gives
\[
\vz_2 = \frac{\mW_1 \vx}{\norm{\mW_1 \vx}/\sqrt{n}}.
\]
The normalization operation at the second layer divides by
\[
\norm{\mW_2 \vz_2}/\sqrt{n} = \norm{\frac{\sqrt{n}\mW_2 \mW_1 \vx}{\norm{\mW_1 \vx}}}/\sqrt{n} = \frac{\norm{\mW_2 \mW_1 \vx}}{\norm{\mW_1 \vx}},
\]
which yields
\[
\vz_3 = \frac{\mW_3 \vz_2}{\norm{\mW_2 \mW_1 \vx}/\norm{\mW_1 \vx}} = \frac{\mW_3 \mW_2 \mW_1 \vx}{\norm{\mW_3 \mW_2 \mW_1 \vx}/\sqrt{n}}.
\]
Continuing this logic yields equation~\eqref{eq:expfn}, where we absorbed $\sqrt{n}$ in the parameter $w_{d+1}$.
%
%

\subsection{Gradient analysis}

Note that with $\mX \in \reals^{n\times n}$ the circulant matrix with first column equal to $\vx$, we have that $\mW_k \vx = \mX \vw_k$ and the channel normalized output of the network becomes
\[
f_N(\mW,\vx) 
= 
w_{d+1}
\frac{ \mX_k \vw_k }{
\norm{ \mX_k \vw_k }
},
\]
where we defined $\mX_k = \prod_{i\neq k} \mW_i \mX$ for notational convenience.
We next compute the gradient
\[
\nabla_{\vw_k} L(\mW,\vx,\vy),
\quad
L(\mW,\vx,\vy)
=
\norm{ \vy - f_N(\mW,\vx) }^2.
\]
Towards this goal, first note that
\[
\nabla_\vz \norm{\vy - \gamma \frac{\vz}{ \norm{\vz} } }^2
=
\frac{\gamma}{\norm{\vz}} \left(
\mI - \vz \transp{\vz} / \norm{\vz}^2
\right)
(\vy - \gamma \vz).
\]
Thus, by the chain rule
\begin{align*}
\nabla_{\vw_k} L(\mW,\vx,\vy)
&=
w_{d+1} \transp{\mX}_k
\frac{1}{\norm{ \mX_k \vw_k }} \left(
\mI - \mX_k \vw_k \transp{\vw}_k \transp{\mX}_k  / \norm{\mX_k \vw_k }^2
\right)
(\vy - w_{d+1} \mX_k \vw_k)\\
&=
\transp{\mX}_k
\frac{w_{d+1}}{\norm{ \mX_k \vw_k }} \left(
\mI - \mX_k \vw_k \transp{\vw}_k \transp{\mX}_k  / \norm{\mX_k \vw_k }^2
\right) \vy.
\end{align*}
From direct computation,
\[
\innerprod{
\nabla_{\vw_k} L(\mW,\vx,\vy)
}{
\vw_k
} = 0,
\]
thus the gradient is orthogonal to $\vw_k$. 
Moreover, if $\vw_k$ is multiplied with a positive scalar $\gamma$, then the gradient scales with $1/\gamma$.


\end{document}